\documentclass[letterpaper]{article}

\usepackage{aistats2020}
\usepackage[margin=1in]{geometry}

\usepackage[utf8]{inputenc} 
\usepackage[T1]{fontenc}    
\usepackage{hyperref}       
\usepackage{url}            
\usepackage{booktabs}       
\usepackage{amsfonts}       
\usepackage{nicefrac}       
\usepackage{microtype}      

\usepackage{amsmath,amsfonts,amssymb,amsthm}

\newtheorem{proposition}{Propostion}
 
\usepackage{booktabs}
\usepackage[load-configurations=version-1]{siunitx} 
\usepackage{tikz}
\usetikzlibrary{shapes,snakes}
\usetikzlibrary{bayesnetID}

\usepackage{bm}
\usepackage{arydshln}
\usepackage{tabularx}

\usepackage[draft]{minted}
\usepackage{fancyvrb}

\usepackage{graphicx}
\usepackage{subcaption}
\usepackage{mwe}

\newcommand*{\underuparrow}[1]{\ensuremath{\underset{\uparrow}{#1}}}

\title{Variationally Inferred Sampling Through a Refined Bound}

          \author{ {\bf Victor Gallego} \\
Institute of Mathematical Sciences (ICMAT) \&\\
SAMSI, Duke University \\
\url{victor.gallego@icmat.es}
\And
{\bf David Rios Insua}  \\
Institute of Mathematical Sciences (ICMAT) \& \\
SAMSI, Duke University \\
\url{david.rios@icmat.es}
}

\usepackage{times}

\begin{document}

\maketitle

%

%






\begin{abstract}
  A framework to boost the efficiency of Bayesian inference in probabilistic programs is introduced by embedding a sampler inside a variational posterior approximation. We call it the refined variational approximation. Its strength lies both in ease of implementation and automatically tuning of the sampler parameters to speed up mixing time using automatic differentiation. Several strategies to approximate \emph{evidence lower bound} (ELBO) computation are introduced. 
  Experimental evidence of its efficient performance is shown solving an influence diagram in a high-dimensional space using a conditional variational autoencoder (cVAE) as a deep Bayes classifier; an unconditional VAE on density estimation tasks; and state-space models for time-series data.
\end{abstract}

\section{INTRODUCTION}
Probabilistic programming offers powerful tools for Bayesian modelling, a framework for describing prior knowledge and reasoning about uncertainty. A probabilistic programming language (PPL) can be viewed as a programming language extended with random sampling and Bayesian conditioning capabilities, complemented with an inference engine that produces answers to inference, prediction and decision making queries. Some examples are WinBUGS \cite{lunn2000winbugs}, Stan \cite{carpenter2017stan}, or the recent Edward \cite{tran2018simple} and Pyro \cite{bingham2018pyro}. The machine learning and artificial intelligence communities are pervaded by models that can be expressed naturally through a PPL. Variational autoencoders (VAE) \cite{kingma2013auto} or  hidden Markov models (HMM) \cite{rabiner1989tutorial} are two relevant examples. 

If we consider a probabilistic program to define a distribution $p(x, z)$, where $x$ are observations and $z$  denote both latent variables and parameters, then we are interested in answering queries involving the posterior $p(z | x)$. This distribution 
is typically intractable but, conveniently, PPLs provide inference engines to approximate this distribution using Monte Carlo methods (e.g., Markov Chain Monte Carlo (MCMC) \cite{andrieu2010particle} or Hamiltonian Monte Carlo (HMC) \cite{neal2011mcmc}) or variational approximations (e.g. Automatic Differentiation Variational Inference (ADVI) \cite{kucukelbir2017automatic}). Whereas the latter are biased and tend to underestimate uncertainty, the former methods may be exceedingly slow depending on the target distribution. For such reason, over the recent years, there has been an increasing interest in developing more efficient posterior approximations \cite{nalisnick2016approximate,salimans2015markov,tran2015variational} and inference engines that aim to be as general and flexible as possible, so they can be used easily for any probabilistic model written as a program \cite{wood2014new, ge2018t}.

It is well known that the performance of a sampling method depends on the parameters used, \cite{papaspiliopoulos2007general}. In this work, we propose a framework to automatically adapt the shape of the posterior and also tune the parameters of a posterior sampler with the aim of boosting Bayesian inference efficiency in probabilistic programs. Our framework can be regarded as a principled way to enhance the flexibility of the variational posterior approximation, yet can be seen also as a procedure to tune the parameters of an MCMC sampler.

Our contributions can be summarised as follows:

\begin{itemize}
    \item A flexible and unbiased variational approximation to the posterior, which consists of improving an initial variational approximation with a stochastic process. An analysis of its key properties is also introduced. 
    \item Several strategies for the ELBO optimization using the previous variational approximation.
\end{itemize}

\subsection{Related work}

The idea of preconditioning the posterior distribution to speed up the mixing time of an MCMC sampler has recently been explored in \cite{hoffmanneutra} and \cite{PhysRevLett.121.260601}, where a reparameterization is learned before performing the sampling via HMC. Both papers extend seminal work in \cite{parno2014transport} by learning an efficient and expressive deep, non-linear transformation instead of a polynomial regression. However, they do not account for tuning the parameters of the sampler as we introduce in Section \ref{sec:main}, where a fully, end to end differentiable sampling scheme is proposed.

The work of \cite{rezende2015variational} introduced a general
framework for constructing more flexible variational distributions, called normalizing flows. These transformations are one of the main techniques to improve the flexibility of current VI approaches and have recently pervaded the literature of approximate Bayesian inference with current developments such as continuous-time normalizing flows \cite{chen2018continuoustime} which extend an initial simple variational posterior with a discretization of Langevin dynamics. However, they require a generative adversarial network (GAN) \cite{goodfellow2014generative} to learn the posterior, which can be unstable in high-dimensional spaces. We overcome this issue with the novel formulation stated in Section \ref{sec:main}. Our framework is also compatible with different optimizers, not only those derived from Langevin dynamics. Other recent proposals to create more flexible variational posteriors are based on implicit approaches, which typically require a GAN \cite{huszar2017variational}, or implicit schema such as UIVI \cite{pmlr-v89-titsias19a} or SIVI \cite{yin2018semi}. Our variational approximation is also implicit, but we use a sampling algorithm to drive the evolution of the density, combined with a Dirac delta approximation to derive an efficient variational approximation, as we report through extensive experiments in Section \ref{sec:exps}.

Closely related to our framework is the work of \cite{hoffman2017learning}, where a VAE is learned using HMC. We use a similar compound distribution as the variational approximation. However, our framework allows for any SG-MCMC sampler (via the entropy approximation strategies introduced) and also the tuning of sampler parameters via gradient descent.
Our work is also related to the recent idea of amortization of samplers \cite{feng2017learning}. A common problem with these approaches is that they incur in an additional error, the so-called amortization gap \cite{cremer2018inference}. We alleviate this by evolving a set of particles $z_i$ with a stochastic process in the latent space after learning a good initial distribution. Hence, the bias generated by the initial approximation is significantly reduced after several iterations of the process. A recent article related to our paper is \cite{pmlr-v97-ruiz19a}, who define a compound distribution similar to our framework. However, we focus on an efficient approximation using the reverse KL divergence, the standard and well understood divergence used in variational inference, which allows for tuning sampler parameters and achieving superior results.


\section{BACKGROUND}
Consider a probabilistic model $p(x|z)$ and a prior distribution $p(z)$ where $x$ denotes an observation and $z \in \mathbb{R}^d$ an unobserved latent variable or parameter, depending on the context. We are interested in performing inference regarding the unobserved variable $\bm{z}$, by approximating its posterior distribution:
$$
p(\bm{z} | \bm{x}) = \frac{ p(\bm{z})p(\bm{x}| \bm{z}) }{ \int p(\bm{z})p(\bm{x}| \bm{z}) d\bm{z} } = \frac{ p(\bm{z})p(\bm{x}| \bm{z}) }{ p(\bm{x}) } = \frac{p(\bm{z},\bm{x})}{p(\bm{x})}.
$$
The previous integral $p(\bm{x}) = \int p(\bm{z})p(\bm{x}| \bm{z}) d\bm{z}$ is typically intractable; no general explicit expressions of the posterior are available. Thus, several techniques have been proposed to perform approximate posterior inference.

\subsection{Inference as optimization}\label{sec:iasopt}

Variational inference, \cite{kucukelbir2017automatic}, tackles the problem of approximating the posterior $p(z | x)$ with a tractable parameterized distribution $q_{\phi}(z|x)$. The goal is to find parameters $\phi$ so that the variational distribution (also referred to as the variational guide or variational approximation) $q_{\phi}(z|x)$ is as close as possible to the actual posterior. Closeness is typically measured through 
Kullback-Leibler 
divergence $KL(q || p)$, which is reformulated into the ELBO, the objective to be optimized using stochastic gradient descent techniques:
\begin{equation}\label{eq:elbo}
\mbox{ELBO}(q) = \mathbb{E}_{q_{\phi}(z|x)} \left[ \log p(x,z) - \log q_{\phi}(z|x)\right].
\end{equation}
Typically, a deep, non-linear model conditioned on observation $x$ defines the mean and covariance matrix of a
Gaussian distribution $q_{\phi}(z|x) \sim \mathcal{N}(\mu_{\phi}(x), \sigma_{\phi}(x))$, to enhance flexibility. 


\subsection{Inference as sampling}

HMC \cite{neal2011mcmc} is an effective sampling method for models whose probability is point-wise computable and differentiable. 
When scalability is an issue, \cite{welling2011bayesian} proposed a formulation of a continuous-time Markov process that converges to a target distribution $p(z | x)$ with $z \in \mathbb{R}^d$. It is based on the Euler-Maruyama discretization of Langevin dynamics:
\begin{eqnarray}\label{eq:sgmcmc}
z_{t+1} \leftarrow z_{t} - \eta_t \nabla \log p(z_t,x)  + \mathcal{N}(0, 2\eta_t I),
\end{eqnarray}
where $\eta_t$ is the step size. The 
required gradient $\nabla \log p(z_t,x)$ can be estimated using mini-batches of data. Several extensions of the original Langevin sampler have been proposed to increase the mixing speed, see for instance \cite{li2016preconditioned,li2016high,abbati2018adageo,gallego2018stochastic}.


\section{THE VARIATIONALLY INFERRED SAMPLING (VIS) FRAMEWORK}\label{sec:main}

In standard VI, the variational approximation $q_\phi(z|x)$ is analytically tractable.
It is typically chosen as a factorized Gaussian distribution as described in Section \ref{sec:iasopt}. 

We propose to use a more flexible approximating posterior by embedding a sampler through:
\begin{equation}\label{eq:q}
q_{\phi,\eta}(z|x) = \int Q_{\eta, T}(z|z_0)q_{0,\phi}(z_0|x)dz_0,
\end{equation}
where $q_{0,\phi} (z | x)$ is the initial and tractable density (i.e., the starting state for the sampler). We will refer to $q_{\phi,\eta}(z|x)$ as the refined variational approximation. The conditional distribution $Q_{\eta, T}(z|z_0)$ refers to a stochastic process parameterized by $\eta$ used to evolve the original density $q_{0,\phi}(z|x)$ and achieve greater flexibility. In the following subsections we describe particular forms of $Q_{\eta, T}(z|z_0)$.
When $T=0$, no refinement steps are performed, so the refined variational approximation coincides with the original variational approximation, $q_{\phi,\eta}(z|x) = q_{0, \phi}(z|x)$. As $T$ increases, the variational approximation will be closer to the exact posterior, provided that $Q_{\eta, T}$ is a valid MCMC sampler.
Next, we maximize a refined ELBO objective, 
\begin{equation}\label{eq:ELBO}
\mbox{ELBO}(q) = \mathbb{E}_{q_{\phi, \eta}(z|x)} \left[ \log p(x,z) - \log q_{\phi, \eta}(z|x)\right]
\end{equation}
to optimize the divergence $KL(q_{\phi,\eta}(z|x) ||  p(z|x))$. The first term of the ELBO only requires sampling from $q_{\phi,\eta}(z|x)$; however the second term, the entropy $-\mathbb{E}_{q_{\phi,\eta}(z|x)} \left[ \log q_{\phi,\eta}(z | x) \right]$ requires also evaluating the evolving, implicit density. 

Regarding $Q_{\eta, T}(z|z_0)$, we consider the following families of sampling algorithms.


\subsection{The sampler $Q_{\eta, T}(z|z_0)$ } \label{sec:grad}
When the latent variables $z$ are continuous ($z \in \mathbb{R}^d$), we evolve the original variational density $q_{0,\phi}(z|x)$ through a stochastic diffusion process. To make it tractable, we discretize the Langevin dynamics using the Euler-Maruyama scheme, arriving at the stochastic gradient Langevin dynamics (SGLD) sampler.
We then follow the process $Q_{\eta,T} (z | z_0)$ (representing $T$ iterations of an MCMC sampler). As an example, for the SGLD sampler $z_i = z_{i-1} + \eta \nabla \log p(x, z_{i-1}) + \xi_{i},$ where $i$ iterates from 1 to $T$; in this case, the only parameter of the SGLD sampler is the learning rate $\eta$. The noise for the SGLD is $\xi_i \sim \mathcal{N}(0, 2\eta I)$. 
The initial variational distribution $q_{0, \phi}(z|x)$ is a Gaussian parameterized by a deep neural network (NN). Then, $T$ iterations of a sampler $Q$ parameterized by $\eta$ are applied leading to $q_{\phi, \eta}$. 

An alternative may be given by ignoring the noise vector $\xi$ \cite{mandt2017stochastic}, thus refining the initial variational approximation with just stochastic gradient descent (SGD).
 Moreover, we can use Stein variational gradient descent (SVGD) \cite{liu2016stein} or a stochastic version \cite{gallego2018stochastic} to apply repulsion between particles and promote a more extensive exploration of the latent space. 

\subsection{Approximating the entropy term}\label{sec:approx}

We propose a set of guidelines for the ELBO optimization using the refined variational approximation.

    \paragraph{Particle approximation (VIS-P).} We can view the flow $Q_{\eta,T} (z | z_0)$ as a mixture of Dirac deltas (i.e., we approximate it with a finite set of particles). That is, we sample $z^1, \ldots, z^K \sim Q_{\eta,T} (z | z_0)$ and use $\tilde{Q}_{\eta,T}(z| z_0) = \frac{1}{K} \sum_{i=1}^K \delta(z - z^i)$. Thus, that entropy term is zero so $\mathbb{E}_{q_{\phi,\eta}(z|x)} \left[ \log q_{\phi,\eta}(z | x) \right] = \mathbb{E}_{q_{0,\phi}(z|x)} \left[ \log q_{0,\phi}(z | x) \right]$. If using SGD as the sampler, the resulting ELBO is tighter than the one with no refinement (see Section \ref{sec:rewriting}). However, discarding the entropy in the sampling process results in variational approximations that are too concentrated around the MAP solution, and this might be undesirable for training generative models. 
    \paragraph{MC approximation (VIS-MC).} Instead of performing the full marginalization in integral (\ref{eq:q}), we can approximate it as $q_{\phi,\eta}(z_T|x) = \prod_{i=1}^T q_\eta(z_i | z_{i-1}) q_{0,\phi}(z_0|x)$. The entropy for each factor can be straightforwardly computed, i.e. for the case of SGLD, $q_\eta(z_i | z_{i-1}) = \mathcal{N}(z_{i-1} + \eta \nabla \log p(x, z_{i-1}), 2\eta I)$. This approximation keeps track of a better estimate of the entropy than the particle approximation.
        \paragraph{Gaussian approximation (VIS-G).} Targeted to settings were it could be helpful to have a posterior approximation that places density over the whole latent space. For the particular case of using SGD as the inner kernel, we have
\begin{align*}
z_0 &\sim q_{0,\phi}(z_0|x) = \mathcal{N}(z_0 | \mu_\phi(x), \sigma_\phi(x))\\
z_i &= z_{i-1} + \eta \nabla \log p(x, z_{i-1}), \qquad i=1,\ldots,T.
\end{align*}
By treating the gradient terms as points, we have that the refined variational approximation can be computed as
$ q_{\phi,\eta}(z|x) = \mathcal{N}(z | z_T, \sigma_\phi(x))$. Note that there is an implicit dependence on $\eta$ through $z_T$.
    \paragraph{Deterministic flows (VIS-D).} If using a deterministic flow (such as SGD or SVGD), we can keep track of the change in entropy at each iteration using the change of variable formula as  in \cite{duvenaud2016early}. However, this requires a costly Jacobian computation, making it unfeasible to combine with our \emph{backpropagation through the sampler} scheme (Sec. \ref{sec:tuning}) for moderately complex problems, so in this work we won't explore this approximation further.
    
    \paragraph{Fokker-Planck approximation (VIS-FP).} Using the Fokker-Planck equation, we can keep track of the density $q_{\phi,\eta}(z|x)$ at each iteration. Then, we may approximate it using a mixture of Dirac deltas. The derivation of this approximation is slightly longer than the previous ones, so we introduce it in Appendix (Supplementary Material) \ref{app:fp}.

\subsection{Tuning sampler parameters}\label{sec:tuning}

In standard VI, the variational approximation $q(z|x;\phi)$ is parameterized by $\phi$. The parameters are learned using SGD or variants such as Adam \cite{kingma2014adam}, using the gradient $\nabla_{\phi} \mbox{ELBO}(q)$. Since we have shown how to embed a sampler inside the variational guide, it is also possible to compute a gradient of the objective with respect to the sampler parameters $\eta$. For instance, we can compute a gradient with respect to the learning rate $\eta$ from the SGLD or SGD process from Section \ref{sec:grad}, $\nabla_{\eta} \mbox{ELBO}(q)$, to search for an optimal step size at every VI iteration. This is an additional step apart from using the gradient $\nabla_{\phi} \mbox{ELBO}(q)$ which is used to learn a good initial sampling distribution.

\section{ANALYSIS OF VIS}

We now highlight and study in detail key properties of the proposed VIS framework.

\subsection{Unbiasedness}

The VIS framework is targeted towards SG-MCMC samplers, where we can compute gradients wrt sampler hyperparameters to speed up mixing time, a common major problem in MCMC.
After backpropagating a few iterations through the SG-MCMC sampler and learning a good initial distribution, one can use the learned sampler normally, in the testing phase, so standard consistency results of SG-MCMC apply as $T \rightarrow \infty$.

\subsection{Refinement of the ELBO}\label{sec:rewriting}

Performing variational inference with the refined variational approximation can be regarded as using the original variational guide while optimizing an alternative, tighter ELBO. Note that for a refined guide of the form $q(z|z_0)q(z_0|x)$, the objective function can be written as
$$
\mathbb{E}_{q(z|z_0)q(z_0|x)} \left[ \log p(x, z) - \log q(z|z_0) - \log q(z_0 | x)\right].
$$
However, using the Dirac Delta approximation for $q(z|z_0)$ and noting that $z = z_0 + \eta \nabla \log p(x,z_0)$ when using SGD and $T=1$, we arrive at the modified objective:
$$
 \mathbb{E}_{q(z_0|x)} \left[ \log p(x, z_0 + \eta \nabla \log p(x,z_0) ) - \log q(z_0 | x)\right]
$$
which is equivalent to the refined ELBO introduced in (\ref{eq:ELBO}). Since we are perturbing the latent variables in the steepest ascent direction, it is straightforward to show that, for moderate $\eta$, the previous bound is tighter than the one, for the original variational guide $q(z_0 | x)$, $\mathbb{E}_{q(z_0|x)} \left[ \log p(x, z_0  ) - \log q(z_0 | x)\right]$. This reformulation of ELBO is also convenient since it provides a clear way of implementing our refined variational inference framework in any PPL supporting algorithmic differentiation.

Respectively, for the VIS-FP case we have that the deterministic flow from VIS-FP follows the same trajectories as SGLD: by standard results of MCMC samplers we have that
$$
KL(q_{\phi,\eta}(z|x) ||  p(z|x)) \leq KL(q_{0, \phi}(z|x) ||  p(z|x)).
$$

\subsection{Taylor expansion}\label{sec:taylor}

From the result in subsection \ref{sec:rewriting},  within the VIS framework, we optimize instead $\max_z \log p(x, z + \Delta z)$, where $\Delta z$ is one iteration of the sampler, i.e., $\Delta z = \eta \nabla \log p(x, z)$ in the SGD case (VIS-P), or $\Delta z = \eta \nabla (\log p(x, z) - \log q(z))$ in the VIS-FP case. For notational clarity, we resort to the case $T=1$, but a similar analysis can be straightforwardly done if more refinement steps are performed.

We may now perform a first-order Taylor expansion of the refined objective as
$$
\log p(x, z + \Delta z) \approx \log p(x, z) + (\Delta z)^\intercal \nabla \log p(x, z).
$$
Taking gradients of the first order approximation w.r.t. the latent variables $z$ we arrive at
$$
\nabla_z \log p(x,z) + \eta \nabla_z \log p(x,z)^\intercal \nabla_z^2 \log p(x,z),
$$
where we have not computed the gradient through the $\Delta z$ term. That is, the \emph{refined gradient} can be deemed as the original gradient plus a second order correction. Instead of being modulated by a constant learning rate, this correction is adapted by the chosen sampler. In the experiments in Section \ref{sec:exp} we show that this is beneficial for the optimization as it can take less iterations to achieve lower losses. By further taking gradients through the $\Delta z$ term, we may tune the sampler parameters such as the learning rate as described in Section \ref{sec:tuning}. Consequently, the next subsection describes both modes of differentiation.


\subsection{Two modes of Automatic Differentiation for the refined ELBO optimization}\label{sec:AD}

Here we describe how to implement two variants of the ELBO objective. First, we define a \emph{stop gradient} operator\footnote{corresponds to \texttt{detach} in Pytorch or \texttt{stop\_gradient} in tensorflow.}  $\bot$ that sets the gradient of its operand to zero, i.e., $\nabla_x \bot (x) = 0$ whereas in the forward pass it acts as the identity function, that is, $\bot (x) = x$. Then, the two variants of the ELBO objective are
\begin{equation}
  \tag{Full AD}
   \mathbb{E}_q \left[ \log p(x, z + \Delta z) - \log q(z + \Delta z | x) \right]
\end{equation} and
\begin{equation}
  \tag{Fast AD}
  \mathbb{E}_q \left[ \log p(x, z + \bot (\Delta z)) - \log q(z + \bot(\Delta z) | x) \right].
\end{equation}
The Full AD ELBO makes it possible to further compute a gradient wrt sampler parameters inside $\Delta z$ at the cost of a slight increase in the computational burden. However, Fast AD variant may be handy in multiple scenarios as we will illustrate in the initial experiments.

\paragraph{Complexity.} Since we need to back propagate through $T$ iterations of an SG-MCMC scheme, using standard results of meta-learning and automatic differentiation \cite{franceschi2017forward}, the time complexity of our more intensive approach (full-AD) is $\mathcal{O}(mT)$, where $m$ is the dimension of the hyperparameters (the learning rate of SG-MCMC and the latent dimension). Since for most use cases the hyperparameters lie on a low-dimensional space, the approach is scalable.

\subsection{Connections with related approaches}

\paragraph{Coupled Variational Bayes (CVB) \cite{dai2018coupled}.}
 In this approach, optimization is in the dual space where we just optimize the standard ELBO. Though if the optimization was exact the solutions would be the same, it is not clear yet what happens in the truncated optimization case (finite $T$), other than performing empirical experiments on given datasets. We thus feel that there is room for implicit methods that perform optimization in the primal space (also they are easier to implement, in a PPL for example). The previous \emph{dual optimization} approach requires the use of an additional neural network (see the CVB paper or \cite{fang2019implicit}). This adds a large amount of parameters and another architecture decision. With VIS we do not need to introduce an auxiliary network, since we perform a "non-parametric" approach by backpropagating instead through $T$ iterations of SGLD. Thus, the only parameters we introduce are the sampler hyperparameters (the step-size in the SGLD case). Also, the lack of an auxiliary network simplifies the design choices.

\paragraph{Contrastive Divergence (CD) \cite{pmlr-v97-ruiz19a}.}

Apart from optimizing the reverse KL divergence (better studied than the divergence  use), the main point is that we can compute gradients wrt sampler parameters $\eta$ (see Section \ref{sec:tuning}), whereas in \cite{pmlr-v97-ruiz19a} the authors only consider a sampler $Q(z|z_0)$: our framework allows for greater flexibility, helping the user in tuning the sampler hyperparameters.







\section{EXPERIMENTS}\label{sec:exps}

We first detail the experiments. We emphasize that our framework permits rapid iterations over a large class of models. Through the following experiments, we aim to shed light on the following questions:
\begin{description}
    \item[Q1] Is the increased computational complexity of computing gradients through sampling steps worth the flexibility gains?
    \item[Q2] Is the proposed framework compatible with other structured inference techniques, such as the sum-product algorithm?
    \item[Q3] Does the more flexible posterior approximated by VIS help in auxiliary tasks, such as decision making or classification?
\end{description}

Within the spirit of reproducible research, the code is released at 
\url{https://github.com/vicgalle/vis}. 
The VIS framework was implemented using Pytorch \cite{paszke2017automatic}, though we also release a notebook for the first experiment using Jax to highlight the simple implementation of the VIS framework.

\subsection{Funnel density}

As a preliminary experiment, we test the VIS framework on a synthetic yet complex target distribution. The target bi-dimensional density is defined through:
\begin{align*}
    z_1 &\sim \mathcal{N}(0, 1.35) \\
    z_2 &\sim \mathcal{N}(0, \exp(z_1)).
\end{align*}
As a variational approximation, we take the usual diagonal Gaussian distribution. For the VIS case, we refine it for $T = 1$ steps using SGLD. Results are shown in Figure \ref{fig:funnel}. In the top, we show the trajectories of the lower bound for up to 50 iterations of variational optimization with Adam. It is clear that our refined version achieves a tighter bound. The middle and bottom figures present the contour curves of the learned variational approximations. The VIS variant is placed nearer to the mean of the true distribution and is more disperse than the original variational approximation, confirming the fact that the refinement step helps in attaining more flexible posterior approximations.

\begin{figure}[!htb]
\begin{center}
\minipage{0.249\textwidth}
\hspace{-0.5em}
  \includegraphics[width=\linewidth]{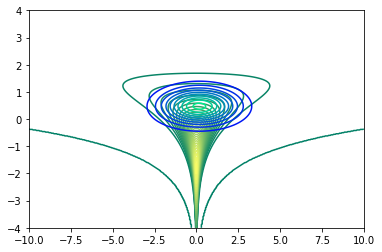}
\endminipage
\minipage{0.249\textwidth}
  \includegraphics[width=\linewidth]{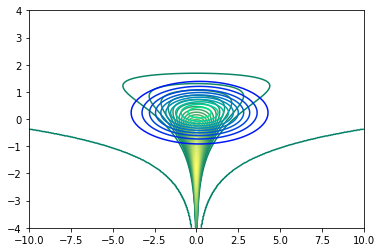}
\endminipage\hfill
\minipage{0.249\textwidth}%
  \includegraphics[width=\linewidth]{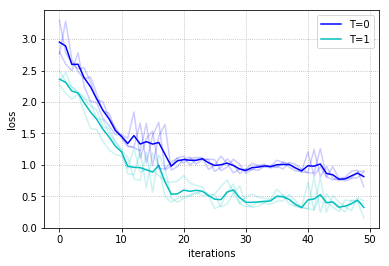}
\endminipage
\end{center}
\caption{Bottom: evolution of the negative ELBO loss objective through 50 iterations. Darker lines depict the mean along different seeds (lighter lines). Top left: contour curves (blue-turquoise) of the variational approximation with no refinement ($T=0$) at iteration 30 (loss of $1.011$). Top right: contour curves (blue-turquoise) of the refined variational approximation ($T=1$) at iteration 30 (loss of $0.667$). Green-yellow curves denote the target density.}\label{fig:funnel}
\end{figure}

\subsection{State-space  Markov models}

We test our variational approximation on two state-space models, one for discrete data and the other for continuous observations. All the experiments in this subsection use the Fast AD version from Section \ref{sec:AD} since it was not necessary to further tune the sampler parameters to have competitive results.

\noindent\textbf{Hidden Markov Model (HMM)}. The model equations are given by
$$
p(z_{1:\tau} , x_{1:\tau}, \theta) = \prod_{t=1}^\tau p(x_t|z_t,\theta_{em})p(x_t|x_{t-1},\theta_{tr})p(\theta),
$$
where each conditional is a Categorical distribution which takes $5$ different classes and the prior $p(\theta) = p(\theta_{em})p(\theta_{tr})$ are two Dirichlet distributions that sample the emission and transition probabilities, respectively. We perform inference on the parameters $\theta$.

\noindent\textbf{Dynamic Linear Model (DLM)}. The model equations are the same as in the HMM case, though the conditional distributions are now Gaussian and the parameters $\theta$ refer to the emission and transition variances. As before, we perform inference over $\theta$. \\

The full model implementations can be checked in Appendix (Supp. Material) \ref{app:ss}, based on \texttt{funsor}\footnote{\url{https://github.com/pyro-ppl/funsor/}}, a PPL on top of the \texttt{Pytorch} autodiff framework. For each model, we generate a synthetic dataset, and use the refined variational approximation with $T = 0, 1, 2$. As the original variational approximation to the parameters $\theta$ we use a Dirac Delta. Performing VI with this approximation corresponds to MAP estimation using the Kalman filter in the DLM case \cite{zarchan2013fundamentals} and the Baum-Welch algorithm in the HMM case \cite{rabiner1989tutorial}, since we marginalize out the latent variables $z_{1:\tau}$. Model details are given in Appendix (Supp. Material) \ref{app:hmm}. Figure \ref{fig:ss} shows the results. The first row reports the experiments related to the HMM; the second one to the DLM. While in all graphs we report the evolution of the loglikelihood during inference, in the first column we report the number of rELBO iterations, whereas in the second column we measure wall-clock time as the optimization takes place. We confirm that VIS ($T>0$) achieve better results than regular optimization with VI ($T=0$) for a similar amount of time.

\begin{figure}[!htb]
\minipage{0.24\textwidth}
  \includegraphics[width=\linewidth]{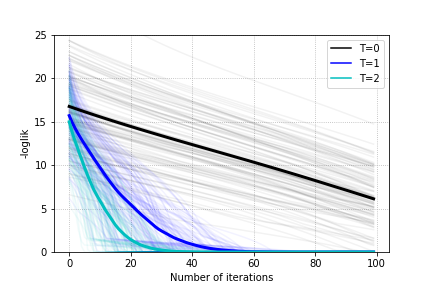}
\endminipage
\minipage{0.24\textwidth}
  \includegraphics[width=\linewidth]{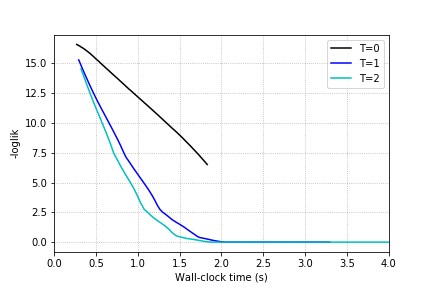}
\endminipage\hfill
\minipage{0.24\textwidth}%
  \includegraphics[width=\linewidth]{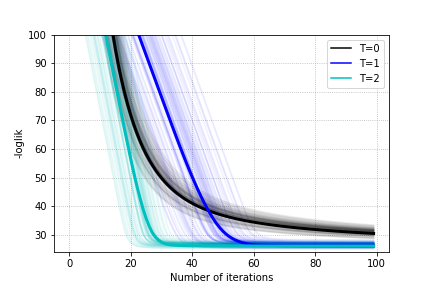}
 \endminipage
 \minipage{0.24\textwidth}%
  \includegraphics[width=\linewidth]{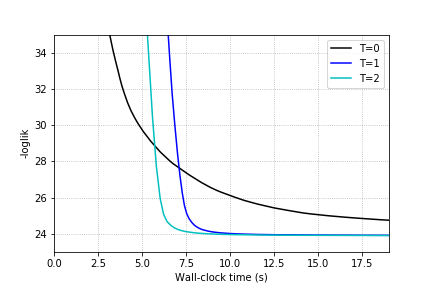}
\endminipage
\caption{Results of rELBO optimization for state-space models. Top left (HMM): -loglikelihood against number of rELBO gradient iterations. Top right (HMM): -loglikelihood against wall-clock time. Bottom left (DLM): -loglikelihood against number of rELBO gradient iterations. Bottom right (DLM): -loglikelihood against number of rELBO gradient iterations}\label{fig:ss}
\end{figure}
    
\subsubsection{Prediction tasks in a HMM} With the aim of assessing whether rELBO optimization helps in attaining better auxiliary scores, we also report results on a prediction task. We generate a synthetic time series of alternating 0 and 1 for $\tau=105$ timesteps. We train the HMM model from before on the first 100 points, and report in Table \ref{tbl:preds} the accuracy of the predictive distribution $p(y_t)$ averaged over the last 5 time-steps. We also report the predictive entropy since it helps in assessing the confidence of the model in its forecast and is a strictly proper scoring rule \cite{gneiting2007strictly}. To guarantee the same computational budget time and a fair comparison, the model without refining is run for 50 epochs, whereas the model with refinement is run for 20 epochs. We see that the refined model achieves higher accuracy than its counterpart; in addition, it is correctly more confident in its predictions.
\begin{table}[!ht]
\centering
\caption{Prediction metrics for the HMM.}\label{tbl:preds}
\begin{tabular}{lcc}
\cline{1-3}
   & $T=0$                             & $T=1$   \\ 
 \cline{1-3}
    accuracy          & $0.40$ &  $\bm{0.84}$ \\
    predictive entropy          & $1.414$ &  $\bm{1.056}$ \\
    logarithmic score   & $-1.044$ & $\bm{-0.682}$ \\
 \cline{1-3}
\end{tabular}
\end{table}

\subsubsection{Prediction task in a DLM}

We now test the VIS framework on the Mauna Loa monthly $CO_2$ time series data \cite{keeling2005atmospheric}. As the training set, we take the first 10 years, and we evaluate over the next 2 years. We use a DLM composed of a local linear trend plus a seasonality block of periodicity 12. Full model specification can be checked in Appendix (Supp. Material) \ref{app:ss}. As a preprocessing step, we standardize the time series to zero mean and unitary deviation. To guarantee the same computational budget time, the model without refining is run for 10 epochs, whereas the model with refinement is run for 4 epochs. We report mean absolute error (MAE) and predictive entropy in Table \ref{tbl:preds_dlm}. In addition, we compute the interval score as defined in \cite{gneiting2007strictly}, a strictly proper scoring rule. As can be seen, for similar wall-clock times, the refined model not only achieves lower MAE, but also its predictive intervals are narrower than the non-refined counterpart.

\begin{table}[!ht]
\centering
\caption{Prediction metrics for the DLM.}\label{tbl:preds_dlm}
\begin{tabular}{lcc}
\cline{1-3}
   & $T=0$                             & $T=1$   \\ 
 \cline{1-3}
    MAE          & $0.270$ &  $\bm{0.239}$ \\
    predictive entropy          & $2.537$ &  $\bm{2.401}$ \\
    interval score ($\alpha=0.05$) & $15.247$ & $\bm{13.461}$\\
 \cline{1-3}
\end{tabular}
\end{table}

\subsection{Variational Autoencoder}

The third batch of experiments aims to check whether the VIS framework is competitive with respect to other algorithms from the recent literature. To this end, we test our approach with a Variational Autoencoder (VAE) model \cite{kingma2013auto}. 
The VAE defines a conditional distribution $p_{\theta}(x | z)$, generating an observation $x$ from a latent variable $z$. For this task, we are interested in modelling two $28 \times 28$ image distributions, MNIST and fashion-MNIST. To perform inference (learn parameters $\theta$) the VAE introduces a variational approximation $q_{\phi}(z | x)$. In the standard setting, this distribution is Gaussian; we instead use the refined variational approximation comparing various values of $T$. We used the MC approximation, though achieved similar results using the Gaussian one. We also use the Full AD variant from Section \ref{sec:AD}.

As experimental setup, we reproduce the setting from \cite{pmlr-v89-titsias19a}. As model $p_{\theta}(x | z)$, we use a factorized Bernoulli distribution parameterized with a two layer feed-forward network with 200 units in each layer and relu activation, except for the final sigmoid activation. As variational approximation $q_{\phi}(z | x)$, we use a Gaussian whose mean and (diagonal) covariance matrix are parameterized by two separate neural networks with the same structure as the previous one, except the sigmoid activation for the mean and a softplus activation for the covariance matrix.

\begin{table}[!ht]
\centering
\caption{Test log-likelihood on binarized MNIST and fMNIST. VIS-$X$-$Y$ denotes $T = X$ refinement iterations during training and $T=Y$ refinement iterations during testing.}\label{tbl:vae}
\begin{tabular}{lcc}
\cline{1-3}
\textbf{Method}   & \textbf{MNIST}                             & \textbf{fMNIST}   \\ \cline{1-3}
 \multicolumn{3}{c}{\small Results from \cite{pmlr-v89-titsias19a}}       \\
    UIVI          & $-94.09$ &  $-110.72$ \\
    SIVI          & $-97.77$ &  $-121.53$ \\
    VAE          & $-98.29$ &  $-126.73$ \\
    \cline{1-3}
 \multicolumn{3}{c}{\small Results from \cite{pmlr-v97-ruiz19a}}       \\
    VCD          & $-95.86$ &  $-117.65$ \\
    HMC-DLGM & $-96.23$ & $-117.74$ \\ 
    \cline{1-3}
    \multicolumn{3}{c}{\small This paper}       \\
    VIS-5-10     & $\bm{-82.74 \pm 0.19}$ & $\bm{-105.08 \pm 0.34}$  \\
    VIS-0-10     & $-96.16 \pm 0.17$ & $-120.53 \pm 0.59$  \\
    VAE (VIS-0-0)              & $-100.91 \pm 0.16$ & $-125.57 \pm 0.63$ \\
\end{tabular}
\end{table}
Results are reported in Table \ref{tbl:vae}. To guarantee a fair comparison, we trained the VIS-5-10 variant for 10 epochs, whereas all the other variants were trained for 15 epochs (fMNIST) or 20 epochs (MNIST), so that the VAE performance is comparable to the one reported in \cite{pmlr-v89-titsias19a}. Although VIS is trained for less epochs, by increasing the number $T$ of MCMC iterations, we dramatically improve on test log-likelihood. In terms of computational complexity, the average time per epoch using $T=5$ is 10.46 s, whereas with no refinement ($T=0$) is 6.10 s (hence our decision to train the refined variant for less epochs): a moderate increase in computing time may be worth the dramatic increase in log-likelihood while not introducing new parameters in the model, except for the learning rate $\eta$. We also show the results from the contrastive divergence approach from \cite{pmlr-v97-ruiz19a} and the HMC variant from \cite{hoffman2017learning}, showing that our framework can outperform those approaches in similar experimental settings. Finally, as a visual inspection of the quality of reconstruction from the VAE trained with the VIS framework, Figure \ref{fig:reco} displays ten random samples of reconstructed digit images.
\begin{figure}[!ht]
  \begin{center}
\includegraphics[width=\linewidth]{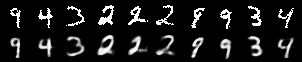}
\end{center}
  \caption{Top row: original images. Bottom row: reconstructed images using VIS-5-10 at 10 epochs.}\label{fig:reco}
\end{figure}

\subsection{Variational Autoencoder as a deep Bayes Classifier}\label{sec:exp}
With the final experiments we show that the VIS framework can deal with more general probabilistic graphical models.
Influence diagrams \cite{howard2005influence} are one of the most popular representations of a decision analysis problem. There is a long history on bridging the gap between influence diagrams and probabilistic graphical models (see \cite{doi:10.1287/opre.36.4.589}, for instance), so developing better tools for Bayesian inference can be transferred to solve influence diagrams.

We showcase the flexibility of the proposed scheme to solve inference problems in an experiment with a classification task in a high-dimensional setting. 
We use the MNIST dataset.
More concretely, we extend the VAE model to condition it on a discrete variable $y \in \mathcal{Y} = \lbrace 0, 1, \ldots, 9 \rbrace$, leading to the conditional VAE (cVAE). A cVAE defines a decoder distribution $p_\theta(x | z, y)$ on an input space $x \in \mathbb{R}^D$ given class label $y \in \mathcal{Y}$ and latent variable $z \in \mathbb{R}^d$. To perform inference, a variational posterior is learned as an encoder $q_\phi(z|x,y)$ from a prior $p(z) \sim \mathcal{N}(0, I)$.
Leveraging the conditional structure on $y$, we use the generative model as a classifier using Bayes rule:
\begin{align}\label{eq:mc_cvae}
p(y|x) \propto p(y)p(x|y) &= p(y) \int p_\theta(x|z,y)q_\phi(z|x,y)dz \ \nonumber \\ &\approx  \frac{1}{K} \sum_{k=1}^K p_\theta (x | z^{(k)}, y)p(y) 
\end{align}
where we use $K$ Monte Carlo samples $z^{(k)} \sim q_\phi(z|x,y)$. In the experiments we set $K = 5$. Given a test sample $x$, the label $\hat{y}$ with highest probability $p(y|x)$ is predicted.
Figure \ref{fig:deep_bayes} in Appendix (Supp. Material) depicts the corresponding influence diagram. Additional details regarding the model architecture and hyperparameters can be found in Appendix (Supp. Material)
\ref{sec:detail}.

For comparison, we perform various experiments changing $T$ for the transition distribution $Q_{\eta, T}$ in the refined variational approximation. 
Results are in Table \ref{tab1}. We report the test accuracy achieved at the end of training. Note that we are comparing different values of $T$ depending on being on the training or testing phases (in the latter, the model and variational parameters are kept frozen). The model with $T_{tr} = 5$ was trained for 10 epochs, whereas the other settings for 15 epochs, to give all settings similar training times.  Results are averaged from 3 runs with different random seeds. From the results, it is clear that the effect of using the refined variational approximation (the cases when $T > 0$) is crucially beneficial to achieve higher accuracy. The effect of learning a good initial distribution and inner learning rate by using the gradients $\nabla_{\phi} \mbox{rELBO}(q)$ and $\nabla_{\eta} \mbox{rELBO}(q)$ has a highly positive impact in the accuracy obtained.

On a final note, we have not included the case when only using a SGD or SGLD sampler (i.e., without learning an initial distribution $q_{0, \phi} (z|x)$) since the results were much worse than the ones in Table \ref{tab1}, for a comparable computational budget. This strongly suggests that for inference in high-dimensional, continuous latent spaces, learning a good initial distribution through VIS can dramatically accelerate mixing time.

\begin{table}
\caption{Results on  digit classification task using a deep Bayes classifier.}\label{tab1}
\centering
\begin{tabular}{llc}
\hline
$T_{tr}$ &  $T_{te}$ & Acc. (test) \\
\hline
0 & 0  & $96.5 \pm 0.5$ \% \\ 
0 & 10 &  $97.7 \pm 0.7$ \%\\
5 & 10 & $ \mathbf{99.8 \pm 0.2}$ \% 
\end{tabular}
\end{table} 

\section{CONCLUSION}

We have proposed a flexible and efficient framework to perform inference in probabilistic programs. We have shown that the scheme can be easily implemented under the probabilistic programming paradigm and used to efficiently perform inference in a wide class of models: state space time series, variational autoencoders and influence diagrams, defined with continuous, high-dimensional distributions.

Our framework can be seen as a general way of tuning MCMC sampler parameters, adapting the initial distributions and learning rate, Section \ref{sec:main}. 
Key to the success and applicability of the VIS framework are the ELBO approximations of the refined variational approximation introduced in Section \ref{sec:approx}, which are computationally cheap but convenient. Better estimates of the refined density and its gradient may be a fruitful line of research, such as the spectral estimator from \cite{shi2018spectral}. Of independent interest to deal with the implicit variational density, it may be worthwhile to consider optimizing the Fenchel dual of the KL divergence, as done recently in \cite{fang2019implicit}. However, this requires the use of an auxiliary neural network, which is a large computational price to pay compared with our lighter particle approximation.


\textbf{Acknowledgments}
VG acknowledges support from grant FPU16-05034. DRI is grateful to the MINECO MTM2017-86875-C3-1-R project and the AXA-ICMAT Chair in Adversarial Risk Analysis. 
All authors acknowledge support from the Severo Ochoa Excellence Programme SEV-2015-0554. 
This material was based upon work partially supported by the National Science Foundation under Grant DMS-1638521 to the Statistical and Applied Mathematical Sciences Institute. 


\bibliographystyle{unsrt}
{\small
\bibliography{plan}
}

\newpage
\clearpage

\appendix

 \section{Fokker-Planck approximation (VIS-FP)}\label{app:fp}
  
  The Fokker-Planck equation is a PDE that describes the temporal evolution of the density of a random variable under a (stochastic) gradient flow. For a given SDE
$$
dz = \mu(z,t)dt + \sigma(z,t)dB_t,
$$
the corresponding Fokker-Planck equation is
$$
\frac{\partial}{\partial t} q_t(z) = -\frac{\partial}{\partial z}\left[ \mu(z,t)q_t(z)\right] + \frac{\partial^2}{\partial z^2} \left[ \frac{\sigma^2(z,t)}{2} q_t(z) \right].
$$

As an example, we are interested in converting the SGLD dynamics to a deterministic gradient flow (that is, we want to convert a SDE into an ODE such that both gradient flows have the same Fokker-Planck equation). 

\begin{proposition}
The SGLD dynamics, given by the following SDE:
$$
dz = \nabla \log p(z)dt + \sqrt{2}dB_t,
$$
have an equivalent deterministic flow, written as the ODE
$$
dz = (\nabla \log p(z) - \nabla \log q_t (z))dt.
$$
\end{proposition}
\begin{proof}
We write the Fokker-Planck equation for the respective flows. For the Langevin SDE, we have
$$
\frac{\partial}{\partial t} q_t(z) = - \frac{\partial}{\partial z} \bigg[ \nabla \log p(z) q_t(z) \bigg] + \frac{\partial^2}{\partial z^2} \bigg[ q_t(z) \bigg].
$$
On the other hand, the Fokker-Planck equation for the deterministic gradient flow is given by
$$
\frac{\partial}{\partial t} q_t(z) = - \frac{\partial}{\partial z} \bigg[ \nabla \log p(z) q_t(z)\bigg] + \frac{\partial}{\partial z} \bigg[ \nabla \log q_t(z) q_t(z)\bigg].
$$
The result immediately follows since $ \frac{\partial}{\partial z} \left[ \nabla \log q_t(z) q_t(z)\right] = \frac{\partial^2}{\partial z^2} \left[ q_t(z) \right]$.
\end{proof}

Given that both flows are equivalent, we restrict our attention to the deterministic one. Its discretization  leads to iterations of the form
\begin{equation}\label{eq:deterministic_flow}
z_{t+1} = z_t - \eta (\nabla \log p(z_t) - \nabla \log q_t (z_t)).
\end{equation}
In order to tackle the last term, we make the following particle approximation. Using a variational formulation, we have that
\begin{align*}
    - \nabla \log q(z) = \nabla \left( - \frac{\delta}{\delta q} \mathbb{E}_q \left[ \log q\right] \right).
\end{align*}
Then, we smoothen the true density $q$ convolving it with a kernel $K$, typically the rbf one, $K(z, z') = \exp \lbrace - \gamma \| z - z' \|^2 \rbrace$, where $\gamma$ is the bandwidth hyperparameter, leading to
\begin{align*}
    \nabla \left( - \frac{\delta}{\delta q} \mathbb{E}_q \left[ \log q\right] \right) &\approx
     \nabla \left( - \frac{\delta}{\delta q} \mathbb{E}_q \left[ \log (q\ast K )\right] \right) \\
     &= \nabla \log (q \ast K) - \nabla \left( \frac{q}{(q \ast K)} \ast K \right).
\end{align*}
If we consider a mixture of Dirac deltas, $q(z) = \frac{1}{K} \sum_{i=1}^K \delta(z - z_i)$, then the approximation is given as
$$
- \nabla \log q(z) \approx - \frac{\sum_k \nabla_{z_i} K(z_i, z_j)}{\sum_j K(z_i, z_j)}
- \sum_k \frac{\nabla_{z_i} K(z_i, z_k)}{\sum_j K(z_j, z_k)},
$$
which we can directly plug into Equation (\ref{eq:deterministic_flow}).
It is possible to backpropagate through Equation (\ref{eq:deterministic_flow}), i.e., the gradients of $K$ can be explicitly computed.

\section{Experiment details}\label{sec:detail}

\subsection{State-space models}\label{app:ss}

\subsubsection{Initial experiments}\label{app:hmm}
For the HMM, both the emission and transition probabilities are Categorical distributions, taking values in the domain $\lbrace 0, 1, 2, 3, 4 \rbrace$.

The equations of the DLM are given by
\begin{align*}
    z_{t+1} &\sim \mathcal{N}(0.5z_t + 1.0,\sigma_{tr}) \\
    x_{t} &\sim \mathcal{N}(3.0z_t + 0.5, \sigma_{em}).
\end{align*}
with $z_0  = 0.0$.


\subsubsection{Prediction task in a DLM}

The DLM model is comprised of a linear trend component plus a seasonal block of period 12. The trend is specified as
\begin{align*}
x_t &= \mu_t + \epsilon_t \qquad \epsilon_t \sim \mathcal{N}(0, \sigma_{obs}) \\
\mu_t &= \mu_{t-1} + \delta_{t-1} + \epsilon'_t \qquad \epsilon'_t \sim \mathcal{N}(0, \sigma_{level}) \\
\delta_t &= \delta_{t-1} + \epsilon''_t \qquad \epsilon''_t \sim \mathcal{N}(0, \sigma_{slope}).
\end{align*}

With respect to the seasonal component,
the main idea is to \emph{cycle the state}: suppose $\theta_t \in \mathbb{R}^p$, with $p$ being the seasonal period. Then, at each timestep, the model focuses on the first component of the state vector:
 $$( \underuparrow{\alpha_1}, \alpha_2, \ldots, \alpha_{p}) \xrightarrow{\text{next period}} ( \underuparrow{\alpha_2}, \alpha_3, \ldots, \alpha_{p}, \alpha_{1}).$$
Thus, we can specify the seasonal component via:
\begin{align*}
x_t &= F\theta_t + v_t \\
\theta_t &= G\theta_{t-1} + w_t
\end{align*}
where $F$ is a $p-$dimensional vector and $G$ is a $p\times p$ matrix such that

\begin{equation*}
G = \begin{bmatrix}
0 & 0 & \ldots & 0 & 1 \\
1 &	0 & & 0 & 0 \\
0 & 1 & & 0 & 0 \\
 & & \ddots & & \\
 0 & 0 & & 1 & 0
\end{bmatrix}
\end{equation*}
and $F = ( 1,0,\ldots, 0,0)$. \\

\subsection{VAE}

\subsubsection{Model details}
\begin{figure}[!htp]
\centering
\RecustomVerbatimEnvironment{Verbatim}{BVerbatim}{}
\begin{minted}{python}
class VAE(nn.Module):
    def __init__(self):
        super(VAE, self).__init__()

        self.z_d = 10
        self.h_d = 200
        self.x_d = 28*28

        self.fc1_mu = nn.Linear(self.x_d, self.h_d)
        self.fc1_cov = nn.Linear(self.x_d, self.h_d)
        self.fc12_mu = nn.Linear(self.h_d, self.h_d)
        self.fc12_cov = nn.Linear(self.h_d, self.h_d)
        self.fc2_mu = nn.Linear(self.h_d, self.z_d)
        self.fc2_cov = nn.Linear(self.h_d, self.z_d)

        self.fc3 = nn.Linear(self.z_d, self.h_d)
        self.fc32 = nn.Linear(self.h_d, self.h_d)
        self.fc4 = nn.Linear(self.h_d, self.x_d)

    def encode(self, x):
        h1_mu = F.relu(self.fc1_mu(x))
        h1_cov = F.relu(self.fc1_cov(x))
        h1_mu = F.relu(self.fc12_mu(h1_mu))
        h1_cov = F.relu(self.fc12_cov(h1_cov))
        # we work in the logvar-domain
        return self.fc2_mu(h1_mu),
        torch.log(F.softplus(self.fc2_cov(h1_cov)))

    def decode(self, z):
        h3 = F.relu(self.fc3(z))
        h3 = F.relu(self.fc32(h3))
        return torch.sigmoid(self.fc4(h3))
\end{minted}
\caption{Model architecture for the cVAE.}
\label{fig:arch_vae}
\end{figure}
The VAE model is implemented with PyTorch \cite{paszke2017automatic}. The prior distribution $p(z)$ for the latent variables $z \in \mathbb{R}^{10}$ is a standard factorized Gaussian. The decoder distribution $p_\theta(x|z)$ and the encoder distribution (initial variational approximation) $q_{0,\phi}(z|x)$ are parameterized by two feed-forward neural networks whose details can be checked in Figure \ref{fig:arch_vae}.

\subsubsection{Hyperparameter settings}
The optimizer Adam is used in all experiments, with a learning rate $\lambda=0.001$. We also set $\eta = 0.001$. We train for 15 epochs (fMNIST) and 20 epochs (MNIST), to achieve similar performance to the explicit VAE case in \cite{pmlr-v89-titsias19a}. For the VIS-5-10 setting, we train for only 10 epochs, to allow for a fair computational comparison (similar computing times).

\subsection{CVAE}


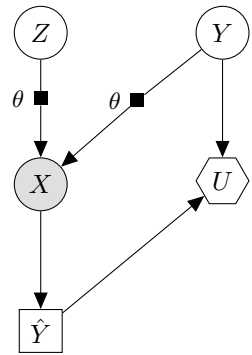
\begin{figure}[ht]
  \begin{center}
\begin{tikzpicture}[x=1.7cm,y=1.8cm] 
\node[latent]    (Z)      {$Z$} ; %
\node[obs, below=of Z, yshift=+0.5cm]                   (X)      {$X$} ; %
 \node[latent, right=of Z]    (Y)      {$Y$} ; %
 \node[utility, below=of Y, yshift=+0.5cm]                   (U)      {$U$} ; %
 \node[decision, below=of X, yshift=+0.5cm]                   (Yh)      {$\hat{Y}$} ; %


   \factor[below=of Z, , yshift=0.3cm]     {Z0-f}     {left:$\theta$} {Z} {X} ; %
    \factor[left=of Y, , yshift=-0.9cm]     {Y0-f}     {left:$\theta$} {Y} {X} ; %

\edge {X} {Yh};
\edge {Y} {U};
\edge {Yh} {U};

\end{tikzpicture}
\end{center}
  \caption{Influence Diagram for the deep Bayes classifier.}\label{fig:deep_bayes}
\end{figure}

\subsubsection{Model details}
\begin{figure}[!htp]
\centering
\RecustomVerbatimEnvironment{Verbatim}{BVerbatim}{}
\begin{minted}{python}
class cVAE(nn.Module):
    def __init__(self):
        super(cVAE, self).__init__()

        self.z_d = 10
        self.h_d = 200
        self.x_d = 28*28
        num_classes = 10

        self.fc1_mu = nn.Linear(self.x_d + num_classes, self.h_d)
        self.fc1_cov = nn.Linear(self.x_d + num_classes, self.h_d)
        self.fc12_mu = nn.Linear(self.h_d, self.h_d)
        self.fc12_cov = nn.Linear(self.h_d, self.h_d)
        self.fc2_mu = nn.Linear(self.h_d, self.z_d)
        self.fc2_cov = nn.Linear(self.h_d, self.z_d)

        self.fc3 = nn.Linear(self.z_d + num_classes, self.h_d)
        self.fc32 = nn.Linear(self.h_d, self.h_d)
        self.fc4 = nn.Linear(self.h_d, self.x_d)

    def encode(self, x, y):
        h1_mu = F.relu(self.fc1_mu(torch.cat([x, y], dim=-1)))
        h1_cov = F.relu(self.fc1_cov(torch.cat([x, y], dim=-1)))
        h1_mu = F.relu(self.fc12_mu(h1_mu))
        h1_cov = F.relu(self.fc12_cov(h1_cov))
        # we work in the logvar-domain
        return self.fc2_mu(h1_mu),
        torch.log(F.softplus(self.fc2_cov(h1_cov)))

    def decode(self, z, y):
        h3 = F.relu(self.fc3(torch.cat([z, y], dim=-1)))
        h3 = F.relu(self.fc32(h3))
        return torch.sigmoid(self.fc4(h3))
\end{minted}
\caption{Model architecture for the cVAE.}
\label{fig:arch}
\end{figure}
The cVAE model is implemented with PyTorch \cite{paszke2017automatic}. The prior distribution $p(z)$ for the latent variables $z \in \mathbb{R}^{10}$ is a standard factorized Gaussian. The decoder distribution $p_\theta(x|y,z)$ and the encoder distribution (initial variational approximation) $q_{0,\phi}(z|x,y)$ are parameterized by two feed-forward neural networks whose details can be checked in Figure \ref{fig:arch}.
The integral (\ref{eq:mc_cvae}) is approximated with 1 MC sample from the variational approximation in all experimental settings.
\subsubsection{Hyperparameter settings}
The optimizer Adam is used in all experiments, with  learning rate $\lambda=0.01$. We set the initial $\eta = 5e-5$.

\end{document}